\documentclass[conference]{IEEEtran}
\IEEEoverridecommandlockouts
\usepackage{cite}
\usepackage{amsmath,amssymb,amsfonts}
\newtheorem{theorem}{Theorem}

\newtheorem{proof}{Proof}[section]
\usepackage{algorithmic}
\usepackage{graphicx}
\usepackage{xcolor}
\def\BibTeX{{\rm B\kern-.05em{\sc i\kern-.025em b}\kern-.08em
    T\kern-.1667em\lower.7ex\hbox{E}\kern-.125emX}}
\begin{document}

\title{Dynamic Decoupling of Placid Terminal Attractor-based Gradient Descent Algorithm}

\author{
\IEEEauthorblockN{Jinwei Zhao}
\IEEEauthorblockA{\textit{Faculty of Computer Science and Engineering} \\
\textit{Xi'an University of Technology}\\
Xi'an, China \\
zhaojinwei@xaut.edu.cn}
\\
\IEEEauthorblockN{Alessandro Betti}
\IEEEauthorblockA{\textit{IMT Scuola Alti Studi}
55100 Lucca, Italy\\
alessandro.betti@imtlucca.it}
\\
\IEEEauthorblockN{Hongtao Zhang}
\IEEEauthorblockA{\textit{Faculty of Computer Science and Engineering} \\
\textit{Xi'an University of Technology}\\
Xi'an, China\\
2640498987@qq.com}
\\
\IEEEauthorblockN{Xinhong Hei}
\IEEEauthorblockA{\textit{Faculty of Computer Science and Engineering} \\
\textit{Xi'an University of Technology}\\
Xi'an, China\\
heixinhong@xaut.edu.cn}
\and
\IEEEauthorblockN{Marco Gori}
\IEEEauthorblockA{\textit{Department of Information Engineering and Mathematics} \\
\textit{University of Siena}\\
Siena, Italy \\
marco.gori@unisi.it}
\\
\IEEEauthorblockN{Stefano Melacci}
\IEEEauthorblockA{\textit{Department of Information Engineering and Mathematics} \\
\textit{University of Siena}\\
Siena, Italy \\
stefano.melacci@unisi.it}
\\
\IEEEauthorblockN{Jiedong Liu}
\IEEEauthorblockA{\textit{Faculty of Computer Science and Engineering} \\
\textit{Xi'an University of Technology}\\
Xi'an, China\\
1041674453@qq.com}
\and

}
\maketitle

\begin{abstract}
Gradient descent (GD) and stochastic gradient descent (SGD) have been widely used in a large number of application domains. Therefore, understanding the dynamics of GD and improving its convergence speed is still of great importance. This paper carefully analyzes the dynamics of GD based on the terminal attractor at different stages of its gradient flow. On the basis of the terminal sliding mode theory and the terminal attractor theory, four adaptive learning rates are designed. Their performances are investigated in light of a detailed theoretical investigation, and the running times of the learning procedures are evaluated and compared. The total times of their learning processes are also studied in detail. To evaluate their effectiveness, various simulation results are investigated on a function approximation problem and an image classification problem.
\end{abstract}


\section{Introduction}

Gradient descent (GD)\cite{boyd2004convex} and stochastic gradient descent (SGD) \cite{Robbins1951},\cite{NIPS2007_0d3180d6} have been widely used to solve many optimization problems in complex machine learning, also in plain machine learning. Understanding the dynamics of GD and SGD is of vital importance\cite{zhang2019gradient} to improve the quality of the optimization processing where the adaptive gradient methods is largely diffused, e.g, Adam\cite{Kingma2014AdamAM}, AdaGrad\cite{duchi2011adaptive}, RMSProp\cite{hinton2012neural}, Nadam\cite{dozat2016incorporating}, AMSGrad\cite{reddi2019convergence}, Radam\cite{liu2019radam}. 

Specifically, we study the dynamics of GD for minimizing a differentiable nonconvex function \(f:\textbf{R}\rightarrow R\), where \(f(w)\) can potentially be stochastic, i.e., \(f(w) = E_x[F(w, x)]\). Such a formula should be able to cover a wide range of problems in machine learning. 

A common characteristic of a GD with the fixed learning rate is that the gradient flow of the GD on the image of the nonconvex function \(f(w)\) incurs an asymptotic convergence\cite{yu2011improved}. It means that the closer to its global minima or its local minima, the slower its convergence speed, while at the positions far from the minimums the convergence may still be slow\cite{yu2011improved}. In addition, if the GD has a small learning rate, it may be easy for its gradient flow  to be trapped in a local minima. To address these problems, many variations of the basic GD have been proposed in recent decades, for example, except of the above adaptive gradient methods, some variations of GD and SGD, such as clipped GD\cite{zhang2019gradient} and normalized GD\cite{arora2022understanding}.  The dynamic analysis on GD and SGD has made great progress, such as, global convergence analysis\cite{lei2019stochastic},\cite{zhang2019gradient},\cite{gower2021sgd} , \cite{khaled2020better}, \cite{mertikopoulos2020almost}, \cite{patel2022stopping}, \cite{patel2022global},\cite{betti2023toward}, local convergence analysis\cite{mertikopoulos2020almost}, greedy and global complexity analysis\cite{gower2021sgd}, \cite{khaled2020better}, asymptotic weak convergence \cite{wang2021convergence}, and saddle-point analysis\cite{fang2019sharp}, \cite{mertikopoulos2020almost}, \cite{jin2021nonconvex}. To implement such an analysis, existing works make a variety of assumptions\cite{patel2022global}, such as (1) the objective function is bounded from below, \(\inf (f) > -\infty\), (2) the gradient function is globally Lipschitz continuous(\(f \in C^{1,1}(\textbf{R}, R)\)) or meets some other relaxed smoothness assumption, (3) the stochastic gradients are unbiased, (4) the variance of the stochastic gradients is bounded. However, we believe that a more complete and valuable conclusion can be obtained by simultaneously decoupling different stages of an arbitrary gradient flow of a GD on the image of the nonconvex function \(f(w)\). Based on the conclusion a valuable optimal method can be designed.

By analyzing the characteristics of different stages of the gradient flow of a GD,  the terminal attractor-based GD algorithm (TAGD) is proposed for using the terminal attractor (TA) to improve its convergence speed and stability.  In the case of local-minima-free error functions, TAGD can guarantee the speed of convergence and the stability of the solution\cite{zak1989terminal}\cite{wang1991terminal} \cite{bianchini1994does}\cite{bianchini1997terminal}. Unfortunately, in the case of multimodal functions, there are no theoretical guarantees that a global solution can be obtained stably unless certain favorable conditions are satisfied \cite{bianchini1994does}. This is due to the fact that the denominator of the formula for the derivative of the weight with respect to time converges more quickly than the numerator of the formula\cite{bianchini1994does}\cite{bianchini1997terminal}.  To improve the stability and speed of the TAGD, some TAs based on the Terminal Sliding Mode (TSM) scheme have been proposed\cite{yu2002fast}\cite{batbayar2007fast}. Recently, a terminal recurrent neural network (RNN) model for time-variant computing has been proposed by featuring ﬁnite-valued activation functions (AFs) based on the terminal attractor. It obtain a ﬁnite-time convergence of error variables \cite{sun2021finitely}\cite{ham2000principles}\cite{zhang2002recurrent}\cite{zhang2005design}\cite{xiao2014different}\cite{li2013accelerating}\cite{xiao2016new}\cite{xiao2017design}\cite{xiao2018solving}\cite{xiao2019performance}\cite{li2019variable}. It achieves a convergence of the error variables in a finite time. These algorithms can speed up the convergence of the back-propagation algorithm when the current solution is far from the local minima and the global minima. However, the internal instability hidden in the weight updating equation based on the terminal attractor still exists\cite{bianchini1997terminal}. These disadvantages reduce the convergence efficiency and convergence speed of the gradient flows of the GD and the SGD.

In the paper, firstly, we carefully analyze the dynamics of the GD and the GD based on the terminal attractor by decoupling the different stages of their gradient flows. Secondly, four adaptive updating methods of the learning rate of the GD are designed based on the terminal attractor. Their performance are analyzed in theory and the total times of their learning processes are investigated based on the terminal sliding mode. The effectiveness of the methods is evaluated by means of various simulation results for some machine learning problems.

\section{Decoupling the gradient flow of an ordinary differential equation of a GD}
 We consider solving a nonconvex optimization problem
 \begin{equation}
 \min_{w \in \textbf{R}=R^d}E(w)
 \label{E(w)}. 
\end{equation}
The derivatives of the objection function, \(E \geq 0\),  with respect to \(w\) are \(\nabla_{w}E\), since  \(E\) is differentiable with respect to \(w\), \(E \in C^{1}\) .  In this paper, the optimal target value of this problem is assumed to be \(E = 0\). If a fixed learning rate \(\gamma\) is given, \(\gamma \neq 0  \wedge \gamma \ll \infty\), and the optimal problem\ref{E(w)} has its unique global optimal solution. By a GD the derivative of \(w\) over time \(t\) (where with "time" we mean the time of the optimization procedure) is given by
\begin{equation}
\frac{dw}{dt}=-\gamma{}{\nabla{}}_{w}E.
\label{ODE}
\end{equation}
The solution of the ordinary differential equation(ODE)\ref{ODE}  form its gradient flow. Along a gradient flow on the image surface of \(E\), the GD can search for its global optimum. If \(E\) is convex, the gradient flow of the GD may contain one global minimum. If \(E\) is a non-convex function,  the gradient flow may contain many local minima as well as a global minimum. Around the local minima or the global minima, because \(\nabla_{w}E\rightarrow 0\), \(\frac{dw}{dt} \rightarrow 0\). So the closer to the global minima or a local minima, the slower the convergence speed of the gradient flow. At positions far from the minima, \({\nabla{}}_{w}E \neq 0\)  and \(\frac{dw}{dt} \neq 0\). However, if \(\gamma\) or \({\nabla{}}_{w}E\) is very small, the convergence may also be very slow regardless of the value of \(E\). 

To improve the convergence speed of the gradient flow and efficiently solve the nonconvex problem (\ref{E(w)}), the GD must perform the following tasks. The first task is that when the current gradient flow is far from the minima of the problem (\ref{E(w)}), its convergence speed should maintain a high value. The second task is that around the local minima of the problem  (\ref{E(w)}), the gradient flow must leave the current position and jump to another position. The third task is that around the global minimum of the problem (\ref{E(w)}), the gradient flow should remain stable in order to asymptotically converge to the global minimum. 

To accomplish these tasks, Zak\cite{zak1989terminal} proposed the theory of  terminal attractors. On the basis of the theory, some researchers have proposed some related learning algorithms \cite{wang1991terminal} \cite{bianchini1994does}\cite{bianchini1997terminal}\cite{batbayar2007fast}.  

\section{Decoupling the gradient flow of a GD based on the terminal attractor \label{Decoupling GF of GDTA}}
For solving the optimization problem (\ref{E(w)}) , the ODE of a GD is designed as 
\begin{equation}
\frac{dw}{dt}=-\gamma{}{\nabla{}}_wE=-\frac{\Omega{}\left(E\right)}{{\left\Vert{}{\nabla{}}_wE\right\Vert{}}^2}{\nabla{}}_wE.
\label{dw/dt}
\end{equation}
Based on the ODE\ref{dw/dt}, the ODE of the function $E$ is
\begin{equation}
\frac{dE}{dt}={\left({\nabla{}}_wE\right)}^T\frac{dw}{dt}=-{\left({\nabla{}}_wE\right)}^T\frac{\Omega{}\left(E\right)}{{\left\Vert{}{\nabla{}}_wE\right\Vert{}}^2}{\nabla{}}_wE=-\Omega{}\left(E\right).
\label{dE/dt)}
\end{equation}
If the function $\Omega{}\left(E\right)=E^{k}, 0<k<1$, using the Taylor expansion, we can find the closed-form solution to the ODE (\ref{dw/dt})\cite{bianchini1997terminal},
\begin{equation}
\tilde{w}\left(t\right)\approx{}\pm{}{\left({\tilde{w}}_0^{2m}-2mh_0t\right)}^{\frac{1}{2m}},
%eq6
\end{equation}
where $\tilde{w}\left(t\right)=w\left(t\right)-w_{min}$, and ${\tilde{w}}_0=\tilde{w}\left(0\right)$, $w_{min}$ is a local minimum of the optimization problem (\ref{E(w)}), \(2m - 1\) is the multiplicity of the first derivative at the minimum and \(h_{0}\) is a constant related to $w_{min}$. It can be found that $\tilde{w}\left(t\right)\not=0$ and the local minimum $w_{min}$ isn't reached until $t=\frac{{\tilde{w}}_0^{2m}}{2mh_0}$.

The Euler's solution of the ODE (\ref{dw/dt}) is
\begin{equation}
w_{n+1}=w_0 - \eta{}\gamma{}{\nabla{}}_wE_n,
%eq7
\end{equation}
where $\eta{} > 0$ is the integration step. When approaching the local minimum $w_{min}$, by the Lipshitz condition \(\left\Vert{}{\nabla{}}_wE\right\Vert{}\rightarrow 0\). Since \(E\neq 0\), $\gamma{}=\frac{E^k}{{\left\Vert{}{\nabla{}}_wE\right\Vert{}}^2} \rightarrow \infty{}$ and the gradient flow is likely to escape the neighborhood of the local minima $w_{min}$. A restart of the $E$ dynamics will also occur in this case.

According to the ODE (\ref{dw/dt}) , it can be seen that for some choices of the power $k$ of $E$, the jumps of the gradient flow can occur even in the neighborhood of the global minimum\cite{bianchini1997terminal}. In fact, according to the Lipshitz condition, \(\left\Vert{}{\nabla{}}_wE\right\Vert{}\rightarrow 0\) may occur when the gradient flow approaches the global minimum, $E \rightarrow 0$. In the ODE\ref{dw/dt}, if the denominator part of the fraction on the right side of the equal sign is a higher order infinitesimal than its numerator part, the ODE\ref{dw/dt} can lead to large increments for the weights \(w\). 

For the one-dimensional case, using Taylor expansions for the numerator and the denominator of  the ODE (\ref{dw/dt}) and discarding the higher order terms, the ODE (\ref{dw/dt}) can be derived as follows\cite{bianchini1997terminal} 
\begin{equation}
\frac{d\tilde{w}}{dt}=O\left({\tilde{w}}^{\left(2mk-2m+1\right)}\right),
\label{ddw/t}
\end{equation}
where $\tilde{w}=w-w_g$, and $2m > 0$ is the multiplicity of the global minimum $w_g$ of the optimization problem (\ref{E(w)}) . From the ODE (\ref{ddw/t}) we see that if $k<1-\frac{1}{2m}$, when $w\rightarrow{}w_g$, $\frac{d\tilde{w}}{dt}\rightarrow{}\infty{}$ will generate a large jump of $w$. The jumping will cause $w$ to leave the global minima.  Even though $\frac{d^{2}E}{dE dt}=-kE^{k-1}=-\infty{}$ at $E=0$ and the Lipshitz condition of $\frac{dE}{dt}$ is violated, the jump transforms a stable equilibrium point into a numerically unstable attractor. Thus, at this moment, the same thing happens to the neighborhood of the global minimum as to the neighborhood of a local minimum. The direction and extent of such jumps are not guaranteed to lead the gradient flow closer to the global minimum $w_g$. The gradient flow may get stuck in the neighborhood of some other local minimum, leading to an oscillating solution for $E$ \cite{bianchini1997terminal}. This cause a instability of the gradient flow around the global minimum.

When the weights \(w\) are far away from these minima, $E \neq 0$ and \(\nabla_{w}E\neq 0\). As \(\nabla_{w}E\) decreases and \(E\) doesn't decrease, according to the ODE (\ref{dw/dt}), \(\frac{dw}{dt}\) may increase which improves the convergence speed of the gradient flow. 

\section{Some adaptive learning rates based on a terminal attractor}
As described in the section\ref{Decoupling GF of GDTA}, the adaptive learning rate based on the terminal attractor can accelerate the convergence of a GD and help the GD to easily jump out of its local minima for solving the optimization problem (\ref{E(w)}).  However, the instability of the gradient flow of the GD  around the global minimum is still there.  The instability will reduce the convergence speed of the GD.  In this section, four learning rates based on the terminal attractor are designed and the gradient flows of the GD and the SGD with them are decoupled  for analyzing their convergence properties. And then, we study the total runtime of the learning processes of the first two learning rates using the terminal sliding modes (TSM) scheme\cite{Venkataraman1992}. The last two learning rates are the same as the runtime formulas of the first two learning rates.
\subsection{A learning rate based on a traditional terminal attractor \label{TA} }
If $\Omega{}\left(E\right)=\beta{}E^{\frac{q}{p}}$, \(\beta{}>0\),   the learning rate is 
\begin{equation}
\gamma{}=\frac{\Omega{}\left(E\right)}{{\left\Vert{}{\nabla{}}_wE\right\Vert{}}^2}=\frac{\beta{}E^{\frac{q}{p}}}{{\left\Vert{}{\nabla{}}_wE\right\Vert{}}^2},
\label{gamma_TA}
\end{equation} 
where $p, q>0$ . For guaranteeing that \(E^{\frac{q}{p}}\)  always is real,  $p$ and $q$ should be positive odd numbers. Certainly, if  \(E \ge 0\), $p$ and $q$ do not have to be odd numbers (This special explanation will not be made later). The learning rate is named as a learning rate based on a traditional terminal attractor (TA). The ODEs of \(w\) and \(E\) of the GD are 
\begin{equation}
\frac{dw}{dt}=-\gamma{}{\nabla{}}_wE=-\frac{\beta{}E^{\frac{q}{p}}}{{\left\Vert{}{\nabla{}}_wE\right\Vert{}}^2}{\nabla{}}_wE,
\label{dw/dt_TA}
\end{equation}
and 
\begin{equation}
    \frac{dE}{dt}={\left({\nabla{}}_wE\right)}^T\frac{dw}{dt}=-\beta{}E^{\frac{q}{p}}.
    \label{dE_dt_TA}
\end{equation}
 
 And then, we can obtain
\begin{equation}
\begin{split}  
\frac{d^{2}E}{dE dt}&= \frac{d}{d E}(-\beta{}E^{\frac{q}{p}})=-\beta{}\frac{q}{p}E^{\frac{q-p}{p}}\\
&=\left\{\begin{array}{
ccc}
0 & \frac{q}{p}>1 \\
-\beta{} & \frac{q}{p}=1 \\
-\infty{} & \frac{q}{p}<1
\end{array}\right.\ \ \ \ \ \ \ \ as\ E\rightarrow{}0.
\end{split}\label{d^2E/dEdt_TA}
\end{equation}

It can be seen that if $0<\frac{q}{p}<1$, at $E\rightarrow 0$, \(\frac{dE}{dt}\) violates the Lipshitz condition(The condition is  \(|\frac{d^{2}E}{dE dt}|<\infty\) in the case.) . At that time, \(\frac{dE}{dt}\) sharply decreases. In this case, the equilibrium point $E=0$ is "infinite" stable on the gradient flow of the GD.  The relaxation time of the gradient flow for the global minimum from any initial point is  
\begin{equation}
\begin{split}
t& = \int_{E(t_{0})}^{0}{\frac{dt}{dE} d E} =- \frac{1}{\beta} \int_{E(t_{0})}^{0}{E^{-\frac{q}{p}} d E}\\
 &= \frac{p}{\beta(p- q)} E^{1-\frac{q}{p}}(t_{0}) < \infty
\end{split}\label{t_TA}
\end{equation}

It represents the global minimum can be intersected by all the attracted transients. Thus, this point is a terminal attractor according to the definition of the terminal attractor in \cite{zak1989terminal}. The following theorem will prove that the optimization process of a SGD or a GD based on the learning rate (\ref{gamma_TA}) will finish in a finite time.
\begin{theorem}
If  the learning rate  $\gamma{}=\frac{\Omega{}\left(E\right)}{{\left\Vert{}{\nabla{}}_wE\right\Vert{}}^2}$, $\Omega{}\left(E\right)=\beta{}E^{\frac{q}{p}}$, \(\beta{}>0\),  the optimization problem is Eq.( \ref{E(w)}), and  $p$ and $q$ are positive odd numbers, 
for solving the optimization problem, a optimization process of a SGD or a GD based on the learning rate (\ref{gamma_TA}) will finish in a finite time.
\label{T1}
\end{theorem}
\begin{proof}
From the above equations (\ref{t_TA}) and (\ref{dE_dt_TA}), we see that $E$ is gradually converging from the initial state $E(t_{0})$ to $E=0$ along a sliding mode manifold $\frac{dE}{dt}+\beta{}E^{\frac{q}{p}}=0$. 

The solution of the ODE (\ref{dw/dt_TA}) become a gradient flow of a SGD or a GD for the optimization problem (\ref{E(w)}). If one step or one epoch of the gradient flow of the SGD or the GD can be considered as a basis of TSM, all gradient flows can be considered as many TSM controls of one order system:
\[
{\dot{s}}_{i}+{\beta{}}_{i}s_{i}^{\frac{q_{i}}{p_{i}}} = 0
,\]
If $s_i=E_i$, ${\beta{}}_i = {\beta{}}$, $q_{i}=q$  and $p_{i}=p$ are positive numbers, and $i=1,...,n$ is an epoch index of the SGD or the GD, we can get
\[{\dot{E}}_{i}+\beta{}E_{i}^{\frac{q}{p}} = 0.\]

By the same analogy, once $E_{n}=0$ is reached  in a finite time, $E_{n-1}$ will reach zero in a finite time, and so will $
E_{n-2},...,E_0$. It is easily seen that after \(n\) epochs the time for reaching the equilibrium point or the global minimum of the optimization problem \ref{E(w)}, , is 
\[
T=\sum_{i=1}^nt_i=\sum_{i=1}^{n}\frac{q}{{\beta{}}\left(p-q\right)}E_{i}^{\frac{p-q}{p}} < \infty
.\]
 Thus, the optimization process  of the SGD or the GD based on the learning rate (\ref{gamma_TA}) 
 will finish in a finite time.
\end{proof}
According to Eq. (\ref{d^2E/dEdt_TA}), to obtain a terminal attractor, $p>q$. When approaching the local minimum $w_{min}$ in the gradient flow, \(\left\Vert{}{\nabla{}}_wE\right\Vert{} \rightarrow 0\) and \(E \neq 0\). Thus, $|\frac{dw}{dt} |\rightarrow  \infty$ according to the ODE\ref{dw/dt_TA}. This will lead to a big jump on the gradient flow. It will even leave the region  around the local minimum $w_{min}$. Near the global minimum, if $q\geq{}\left(1-\frac{1}{2m}\right)p$, according to Eq. (\ref{ddw/t}) (because \(k=\frac{q}{p}\)), $\frac{dw}{dt}\rightarrow 0$. According to Eq. (\ref{d^2E/dEdt_TA}) because \(E\rightarrow{0}\), the optimization path will approach the stable equilibrium global minima. On the other hand, if $q < \left(1-\frac{1}{2m}\right)p$, the aimless jump of $w$ will occur, destroying the convergence efficiency. When the weights \(w\) are far from these minima, if \(\left\Vert{}{\nabla{}}_wE\right\Vert{}\) decreases and \(E\) doesn't decrease, according to Eq. (\ref{dw/dt_TA}), \(\frac{dw}{dt}\) will increase, which improves the convergence speed.  
\subsection{A learning rate based on a fast terminal attractor \label{FTA} }
When the current Euler solution of the optimization problem (\ref{E(w)}) is far away from the equilibrium, the learning rate based on a traditional terminal attractor (\ref{gamma_TA}) does not prevail over the linear counterpart (setting \(p = q\)) \cite{yu2002fast}. One immediate solution is to introduce the following so-called fast terminal attractor for definiting a learning rate\cite{yu2002fast}, namely as a learning rate based on a fast terminal attractor (FTA):
\begin{equation}
\gamma{}=\frac{\Omega{}\left(E\right)}{{\left\Vert{}{\nabla{}}_wE\right\Vert{}}^2}=\frac{\alpha{}E+\beta{}E^{\frac{q}{p}}}{{\left\Vert{}{\nabla{}}_wE\right\Vert{}}^2},
\label{gamma_FTA}
\end{equation}
where $\Omega{}\left(E\right)=\alpha{}E+\beta{}E^{\frac{q}{p}}$,  and  $p, q>0$.  The ODE of \(w\) and \(E\) of a GD is 
\begin{equation}
\frac{dw}{dt}=-\gamma{}{\nabla{}}_wE=-\frac{\alpha{}E+\beta{}E^{\frac{q}{p}}}{{\left\Vert{}{\nabla{}}_wE\right\Vert{}}^2}{\nabla{}}_wE,
\label{dw/dt_FTA}
\end{equation}
and
\begin{equation}
    \frac{dE}{dt}=-\left(\alpha{}E+\beta{}E^{\frac{q}{p}}\right).
    \label{dE_dt_FTA}
\end{equation}
And then, we can obtain
\begin{equation}
\begin{split} 
\frac{d^{2}E}{dEdt}&=-\left(\alpha{}+\beta{}\frac{q}{p}E^{\frac{q-p}{p}}\right)  \\
&=\left\{\begin{array}{
ccc}
-\alpha{} & \frac{q}{p}>1 \\
-\left(\alpha{}+\beta{}\right) & \frac{q}{p}=1 \\
-\infty{} & \frac{q}{p}<1
\end{array}\right.\ \ \ \ \ \ \ \ as\ E\rightarrow{}0.
\label{d^2E/dEdt_FTA}
\end{split}
\end{equation}

It can be seen that if $0<\frac{q}{p}<1$,  \(\frac{dE}{dt}\) violates the Lipshitz condition.  The equilibrium point $E=0$ is "infinite" stable on the gradient flow of the GD and $E$~gradually converge to $E=0$~from the initial state $E_0$ along a sliding mode manifold  $\frac{dE}{dt}+\alpha{}E+\beta{}E^{\frac{q}{p}}=0$. $E=0$ is a terminal attractor. The convergence time  \(t\) is
\begin{equation}
t=\frac{p}{\alpha{}\left(p-q\right)}ln\left(\frac{\alpha{}}{\beta{}}{E_0}^{\frac{p-q}{p}}+1\right)< \infty.
\end{equation}
Just like the theorem\ref{T1}, for solving the optimization problem (\ref{E(w)}), one step or one epoch of the gradient flow of a SGD or a GD can be considered as a basis of TSM. All epochs can be considered as many TSM controls of one order system:
\begin{equation}
 {\dot{E}}_{i}+{\alpha{}}E_{i}+{\beta{}}E_{i}^{\frac{q}{p}}  = 0
\end{equation}
where ${\alpha{}}>0$, ${\beta{}}>0$, and $p, q>0$. By the same analogy, once $E_{n}=0$ is reached, $E_{n-1}$ will reach zero in finite time, and so will $E_{n-1},...,E_0$. It is easy to see that the time to reach the equilibrium point, \(E_{i} \rightarrow 0\), is
\begin{equation}
\begin{split}
T=\sum_{i=1}^nt_i=\sum_{i=1}^{n}\frac{q}{{\alpha{}}\left(p-q\right)}  
ln\left({\alpha{}}{\beta{}}E_{i}^{\frac{p-q}{p}}+1\right)< \infty.
\end{split}
\end{equation}
Thus, the optimization process  based on the SGD or the GD will finish in a finite time.

In the neighborhood of the local minimum $w_{min}$, \(\left\Vert{}{\nabla{}}_wE\right\Vert{} \rightarrow 0\) and \(E \neq 0\). Thus, according to Eq.(\ref{dw/dt_FTA}),$|\frac{dw}{dt}| \rightarrow \infty$. This will lead to a large jump of the gradient flow. It will even leave the current region  around the local minimum $w_{min}$. Near the global minimum, according to Eq.(\ref{dw/dt_FTA}), because the denominator of the formula for the derivative of the weight with respect to time converges more quickly than the numerator of the formula, $\frac{dw}{dt}\rightarrow 0$. It can also be found that the FTA can relax the condition of TA for $\frac{dw}{dt}\rightarrow 0$ near the global minimum. Because \(E\rightarrow{0}\), based on Eq.(\ref{d^2E/dEdt_FTA}), when \(\frac{q}{p}<1\), \(\frac{d^{2}E}{dEdt} \rightarrow - \infty\). Thus, the gradient flow approaches a stable global minima. When the weights \(w\) are far from these minima, if \(\left\Vert{}{\nabla{}}_wE\right\Vert{}\) decreases and \(E\) doesn't decrease, according to Eq.(\ref{dw/dt_FTA}), \(\frac{dw}{dt}\) will increase, which improves the convergence speed of the gradient flow. Because $\Omega{}\left(E\right)=\alpha{}E+\beta{}E^{\frac{q}{p}}$ ($\frac{q}{p}>0$), its convergence speed \(\frac{dE}{dt}\)should be faster than that of  the GD or the SGD with a learning rate based on a traditional terminal attractor under the same conditions.

\subsection{A learning rate based on a placid terminal attractor and a learning rate based on a placid fast terminal attractor}
According to the ODE (\ref{dw/dt_TA}) and the ODE (\ref{dw/dt_FTA}), it can be seen that if  \(\nabla_{w}E\rightarrow 0\), \(\frac{dw}{dt}\) and the solutions of the ODEs often fall into infinity. The infinity value of \(w\) will cause the gradient flow jump to a uncontrollable position which will destroy its stability and reduce the speed of its convergence.  In order to ensure the finite value of \(w\), we propose a new learning rate with an upper bound based on TA, namly a learning rate based on a placid terminal attractor(PTA),
\begin{equation}
\begin{split}
\gamma{}=\frac{\beta{}E^{\frac{q}{p}}}{{\left\Vert{}{\nabla{}}_{w}E\right\Vert{}}}\delta\left(\frac{1}{\left\Vert{}{\nabla{}}_{w}E\right\Vert{}}\right),
\end{split}
\end{equation}
where  $p, q>0$ and \(\delta(\cdot)\) is the Sigmoid function. Through the following analysis and discussion, it can also be found that the PTA can relax the condition for $\frac{dw}{dt}\rightarrow 0$ near the global minimum. The ODE of \(w\) is 
\begin{equation}
\frac{dw}{dt}=-\frac{\beta{}E^{\frac{q}{p}}}{{\left\Vert{}{\nabla{}}_wE\right\Vert{}}}\delta\left(\frac{1}{\left\Vert{}{\nabla{}}_{w}E\right\Vert{}}\right){\nabla{}}_wE 
\label{dw/dt_PTA}
\end{equation}
and the ODE of \(E\) is 
\begin{equation}
\frac{dE}{dt}=-{\left\Vert{}{\nabla{}}_wE\right\Vert{}}{\beta{}E^{\frac{q}{p}}}\delta\left(\frac{1}{\left\Vert{}{\nabla{}}_{w}E\right\Vert{}}\right).
\label{dE/dt_PTA}
\end{equation}
Further, we can obtain 
\begin{equation}
\frac{d^{2}E}{dEdt}=-{\left\Vert{}{\nabla{}}_wE\right\Vert{}}\left(\beta{}\frac{q}{p}E^{\frac{q-p}{p}}\right)\delta\left(\frac{1}{\left\Vert{}{\nabla{}}_wE\right\Vert{}}\right),
\end{equation}
and

\begin{equation}
\beta{}\frac{q}{p}E^{\frac{q-p}{p}} 
=\left\{\begin{array}{
ccc}
0 & \frac{q}{p}>1 \\
\beta{} & \frac{q}{p}=1 \\
\infty{} & \frac{q}{p}<1
\end{array}\right . \ \ \ \ \ \ \ \ as\ E\rightarrow{}0.
\label{d^2E/dEdt_PTA}
\end{equation}
In accordance with the setting of the terminal attractor, $p>q$. 

For improving the rate of the convergence of the gradient flow of a GD or a SGD at the position far away from the minimums, a learning rate based on a placid fast terminal attractor, namly a learning rate based on a placid fast terminal attractor (PFTA), is proposed as follows as
\begin{equation}
\begin{split}
\gamma{}=\frac{\alpha{}E+\beta{}E^{\frac{q}{p}}}{{\left\Vert{}{\nabla{}}_wE\right\Vert{}}}\delta\left(\frac{1}{\left\Vert{}{\nabla{}}_wE\right\Vert{}}\right), 
\end{split}
\end{equation}
where $p, q>0$ and \(\delta(\cdot)\) is the Sigmoid function. Similar to the previous learning rate, we can get the following ODEs of \(w\) and \(E\),
\begin{equation}
\frac{dw}{dt}=-\frac{\alpha{}E+\beta{}E^{\frac{q}{p}}}{{\left\Vert{}{\nabla{}}_wE\right\Vert{}}}\delta\left(\frac{1}{\left\Vert{}{\nabla{}}_wE\right\Vert{}}\right){\nabla{}}_wE ,
\label{dw/dt_PFTA}
\end{equation}
and 
\begin{equation}
\begin{split} 
\frac{dE}{dt}&=-{\left\Vert{}{\nabla{}}_wE\right\Vert{}}({\alpha{}E+\beta{}E^{\frac{q}{p}}})\delta\left(\frac{1}{\left\Vert{}{\nabla{}}_wE\right\Vert{}}\right).
\label{dE/dt_PFTA}
\end{split}
\end{equation}

Further, we can obtain 
\begin{equation}
\frac{d^{2}E}{dEdt}=-{\left\Vert{}{\nabla{}}_wE\right\Vert{}}\left(\alpha{}+\beta{}\frac{q}{p}E^{\frac{q-p}{p}}\right)\delta\left(\frac{1}{\left\Vert{}{\nabla{}}_wE\right\Vert{}}\right)
\end{equation}
and
\begin{equation}
\alpha{}+\beta{}\frac{q}{p}E^{\frac{q-p}{p}} 
=\left\{\begin{array}{
ccc}
\alpha{} & \frac{q}{p}>1 \\
\left(\alpha{}+\beta{}\right) & \frac{q}{p}=1 \\
\infty{} & \frac{q}{p}<1
\end{array}\right .\ \ \ \ \ \ \ \ as\ E\rightarrow{}0.
\label{d^2E/dEdt_PFTA}
\end{equation}
Let's compare and analyze their properties by decoupling different stages of the gradient flows of a GD based on the both learning rates. When approaching the local minimum $w_{min}$, \(\left\Vert{}{\nabla{}}_wE\right\Vert{}\rightarrow 0\), \(E \neq 0\), and \(\delta\left(\frac{1}{\left\Vert{}{\nabla{}}_wE\right\Vert{}}\right)\rightarrow 1\).  Thus, according to the ODE \ref{dw/dt_PTA}, $|\frac{dw}{dt}| \rightarrow \beta{}E^{\frac{q}{p}}$ 
, and according to the ODE\ref{dw/dt_PFTA}, $|\frac{dw}{dt}| \rightarrow \alpha{}E+\beta{}E^{\frac{q}{p}}$.  

If \(E\)  has a big value, the gradient flow will have a big step forward \(|\frac{dw}{dt}| \). The steps of the both gradient flows based on PTA and PFTA do not equal  the infinity, which prevents the gradient flow from jumping out of control. The step size of PFTA is greater than one of PTA. That means that the gradient flow based on PFTA is faster than one based on PTA at that case. 

In the neighborhood of the global minimum of the problem (\ref{E(w)}),  although \(\left\Vert{}{\nabla{}}_wE\right\Vert{}\rightarrow 0\) ,  \(\delta\left(\frac{1}{\left\Vert{}{\nabla{}}_wE\right\Vert{}}\right)\rightarrow 1\) and \(\delta\left(\frac{1}{\left\Vert{}{\nabla{}}_wE\right\Vert{}}\right)\) has  an upper bound, 1. According to Eq.(\ref{dw/dt_PTA}) and  Eq. (\ref{dw/dt_PFTA}), as  \(E\rightarrow 0\) , $|\frac{dw}{dt}| \rightarrow 0$.  This result does not require any restrictions to be true, unlike TA. According to Eq.(\ref{d^2E/dEdt_PTA})  and Eq.(\ref{dE/dt_PTA}), the gradient flow will approach a stable equilibrium global minima because \(\frac{dE}{dt} \rightarrow 0\). The gradient flow of the GD based on PFTA also has the same properties around the global minimum with one of the GD based on PTA. 

If the weights \(w\) are far from these minima, \(\left\Vert{}{\nabla{}}_wE\right\Vert{}\) decreases, and \(E\) doesn't decrease, according to Eq.(\ref{dw/dt_PTA}), \(|\frac{dw}{dt}|\) will increase, which improves the convergence speed. At that time, according to Eq.(\ref{dw/dt_PFTA})  \(|\frac{dw}{dt}|\) will also increase. Because $\alpha{}E+\beta{}E^{\frac{q}{p}} > \beta{}E^{\frac{q}{p}}$
($\frac{q}{p}>0$), its convergence speed \(\frac{dE}{dt}\)should be faster than that of PTA. Based on the same analysis as in the previous two sections\ref{TA} and \ref{FTA}, the optimization process of SGD or GD based on PTA or PFTA must finish in a finite time for solving the problem (\ref{E(w)}).
\section{Experiments}
In this section, the proposed PFTA, and PTA are illustrated by the applications to the function approximation and the image classification. PFTA and PTA are applied to the SGD for adaptively modifying its learning rate, respectively. In the experiments they are compared with some popular gradient based optimization algorithms with other learning rates including some adaptive learning rates, such as  SGD\cite{NIPS2007_0d3180d6},  Adam\cite{Kingma2014AdamAM}, AdaGrad\cite{duchi2011adaptive}, RMSProp\cite{hinton2012neural}, Nadam\cite{dozat2016incorporating}, AMSGrad\cite{reddi2019convergence}, Radam\cite{liu2019radam}), conjugate gradient descent (CGD)\cite{hestenes1952methods}, limited memory BFGS (LBFGS)\cite{liu1989limited}.
\subsection{Function approximation}
In this experiment, to verify our theoretical findings, we first checked their behavior for training a simple neural network to approximate a simple binary nonlinear function
\begin{equation}
y = sin(x^{2}_{1})+sin(x^{2}_{2}).
\end{equation}
For efficient comparison, the same network topology is used in all simulations. The neural network has a single hidden layer. The hidden layer consists of five neurons. The hidden neurons use the rectified linear activation function, while the output neurons use the linear function. The Mean Square Error (MSE) is used to evaluate the behavior and performance of the algorithms. A set of 100 samples is used for the experiments. For each learning algorithm, 100 simulations are performed to obtain the average MSE curves during the training process. When the MSE of a simulation reaches \(10^{-4}\), the simulation is stopped. The maximum number of iterations is set as 800. The training data set is normalized before carrying out the experiment. 

In PFTA, the parameters \(\alpha=0.03,\beta=0.1\) and \(\frac{q}{p}=0.65\). In PTA, \(\eta=0.09\) and \(\frac{q}{p}=0.7\). Other algorithms use their optimal super-parameter settings obtained by the grid searching method as Table.\ref{tab:Table-super parameters}. 
\begin{table*}
\centering
\begin{tabular}{|c|c|c|c|c|c|c|c|}
\hline
\textbf{Algorithm}& \textbf{RMSprop} &\textbf{LBFGS}& \textbf{Adam} &\textbf{Radam} & \textbf{Nadam}& \textbf{ASGD}&\textbf{SGD}\\
\hline
learning rate&  0.095&0.045& 0.055&0.08& 0.09& 0.065&0.04\\
\hline
\end{tabular}
      \vspace{10pt}
               \caption{\label{citation-guide}Super-parameter settings of all contrast algorithms.}
\label{tab:Table-super parameters}
\end{table*}

The average MSE curves of all algorithms are shown in Fig. \ref{fig:EXP-1-All}. Table \ref{tab:Table-runtime} lists the runtime required by each algorithm to complete an experiment.
\begin{figure}
    \centering
    \includegraphics[width=0.8\linewidth]{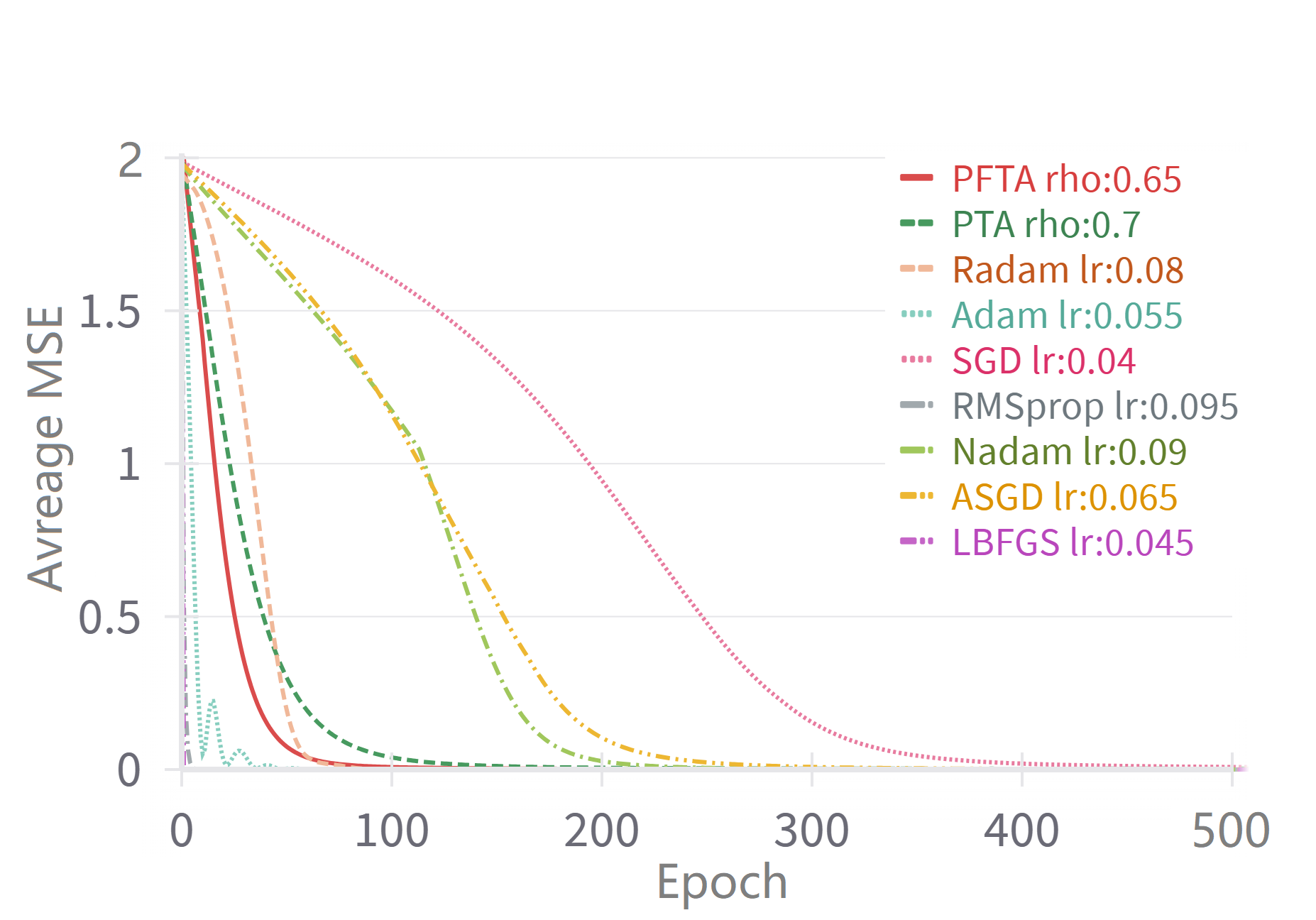}
    \caption{Performances of vary training algorithms}
    \label{fig:EXP-1-All}
\end{figure}
\begin{figure}
    \centering
    \includegraphics[width=0.8\linewidth]{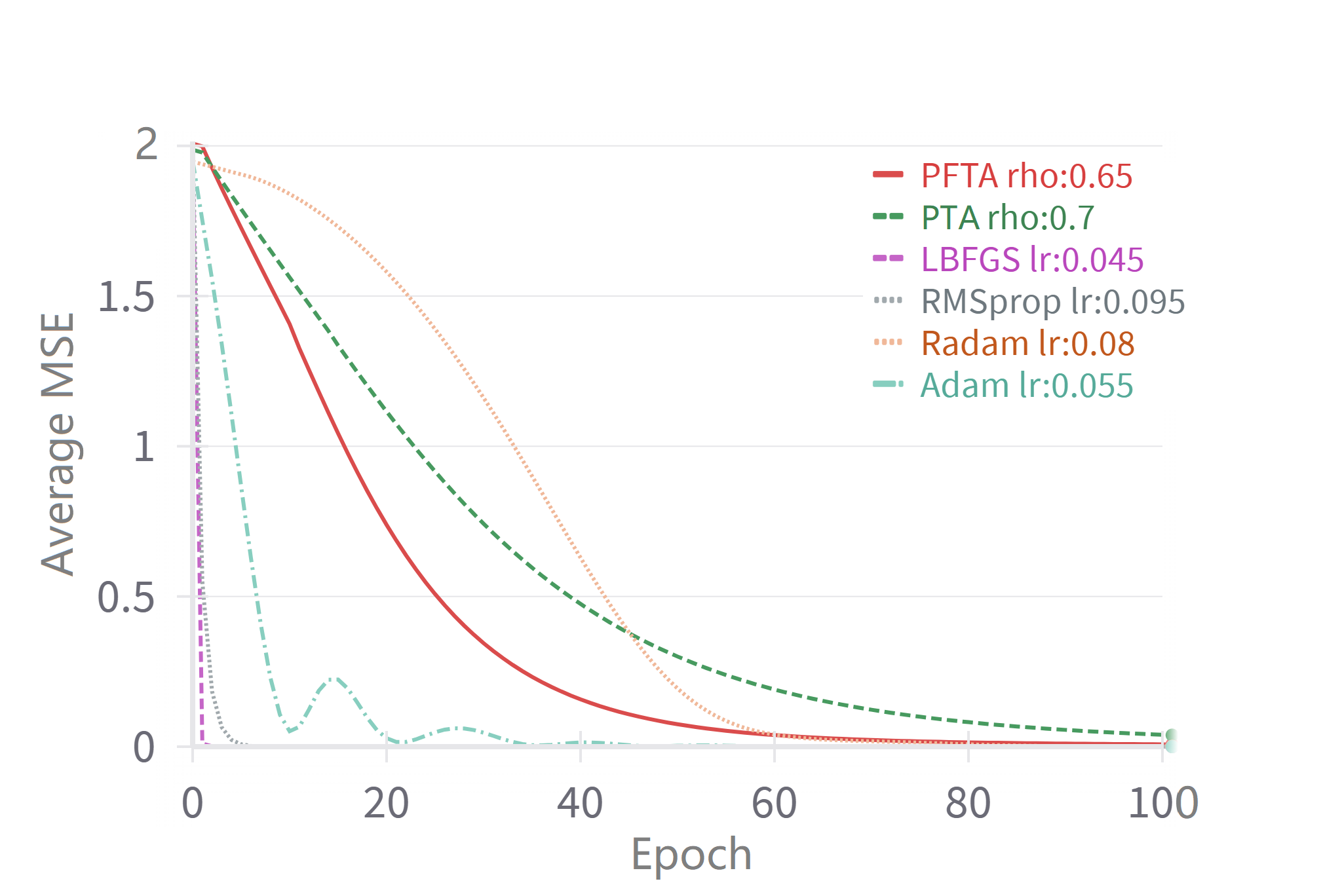}
    \caption{Performances of PFTA, PTA, Radam, Adam, SGD, LBFGs}
    \label{fig:Exp-1-PFTA-PTA-ADAM-RADAM-RMS-LB}
\end{figure}
From Fig. \ref{fig:EXP-1-All}, one can see PFTA and FTA achieve a rather satisfying performance. It can be seen that LBFGS, Adam and RMSprop convergence to the target with fewer epochs than PFTA and FTA. However, we find that Adam has some jumps before it converges to the optimal point. It indicates that its gradient flow is unstable. From Table.\ref{tab:Table-runtime}, we can see that LBFGS takes more runtime than RMSprop and PFTA for the training process. PFTA takes less runtime than FTA, which means it converges faster.  PFTA and RMSprop spend about the same amount of running time.  We know RMSprop adaptively use the past gradients for tuning the learning rate\cite{reddi2019convergence}. This proves that the strategy based on the terminal attractor is correct and the “long-term memory” of  of past gradients is import in the experiment. Future, we will consider the  “long-term memory” strategy for improving PFTA and PTA . In Fig.\ref{fig:Exp-1-PFTA-PTA-ADAM-RADAM-RMS-LB}  we compare the MSE curves of the 6 best algorithms. We can see the detail of the curves.
\begin{table*}
\centering
\begin{tabular}{|c|c|c|c|c|c|c|c|c|c|}
\hline
\textbf{Algorithm}& \textbf{RMSprop} &\textbf{LBFGS}& \textbf{Adam} &\textbf{Radam} & \textbf{Nadam}& \textbf{ASGD}&\textbf{SGD}& \textbf{PTA} &\textbf{PFTA}\\
\hline
Runtime&  15s&43s& 28s&13s & 2m22s& 16s&1m 56s& 2m 16s&17s\\
\hline
\end{tabular}
 \vspace{10pt}
               \caption{\label{citation-guide}Runtime in a training process.}
\label{tab:Table-runtime}
\end{table*}
\subsection{Image classification}
We completed an image classification task on the standard CIFAR-10 dataset\cite{cifar-10} using  ResNet-34 \cite{he2016deep}. In this experiment, we use 200 epochs for the  training procedure. We compare PFTA, PTA, and variants with popular optimization methods such as SGD, AdaGrad, Adam, AMSGrad, AdaBound\cite{luo2019adaptive}, and AMSBound\cite{luo2019adaptive}.  Each experiment is run three times from random starts. We choose the experiment results with the lowest training loss at the end for every algorithm. Except for PFTA and FTA,  all algorithms are reduced the learning rates by 10 after 150 epochs according with the experiment setting in the reference\cite{luo2019adaptive}.

The training accuracy and testing accuracy of every algorithm on ResNet-34 model are shown  in Fig. \ref{fig:Exp-2-Resnet-trainging} and Fig.\ref{fig:Exp-2-Resnet-testing} respectively.  It can be seen that from the very beginning, the accuracy of PFTA and PTA surpass all algorithm. Their optimal solutions have been reached at the 50th epoch. After the 50th epoch, they have been steadily rising slowly, with almost no fluctuation. After 150 epoches, although super parameters of PFTA and PTA aren't modified, their performance still are very excellent. We think that it is because at the moment \(\delta(\frac{1}{\left\Vert{}{\nabla{}}_wE\right\Vert{}})\) has  an upper bound, according to Eq.(\ref{dw/dt_PTA}) and Eq.(\ref{dw/dt_PFTA}), as  \(E\rightarrow 0\) , $|\frac{dw}{dt}| \rightarrow 0$. When \(\left\Vert{}{\nabla{}}_wE\right\Vert{} \rightarrow 0\), \(\frac{d^{2}E}{dEdt} \rightarrow -\infty{}\), the current point become a terminal attractor. The results show that the character of  the terminal attractor in PTA and PFTA can ensure the stability of the final stage of their gradient flow of the SGD and improve their convergence speed at the other stages.   
\begin{figure}
    \centering
    \includegraphics[width=0.8\linewidth]{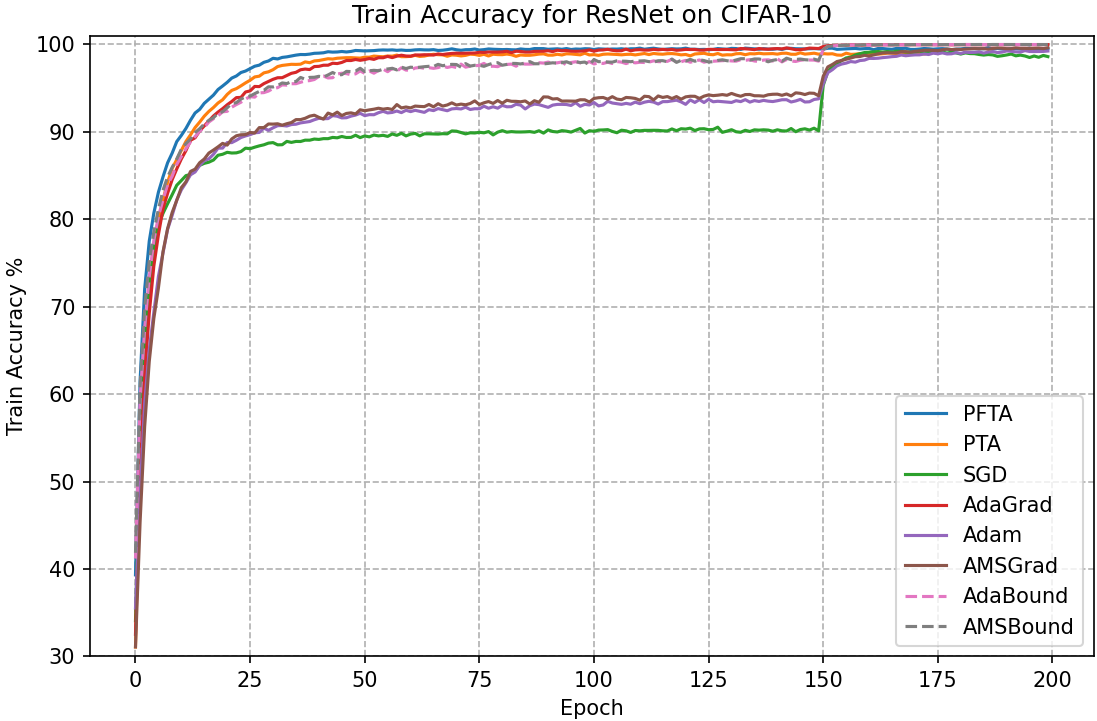}
    \caption{Training Accuracy for ResNet on CIFAR-10 dataset}
    \label{fig:Exp-2-Resnet-trainging}
\end{figure}
    \begin{figure}
        \centering
        \includegraphics[width=0.8\linewidth]{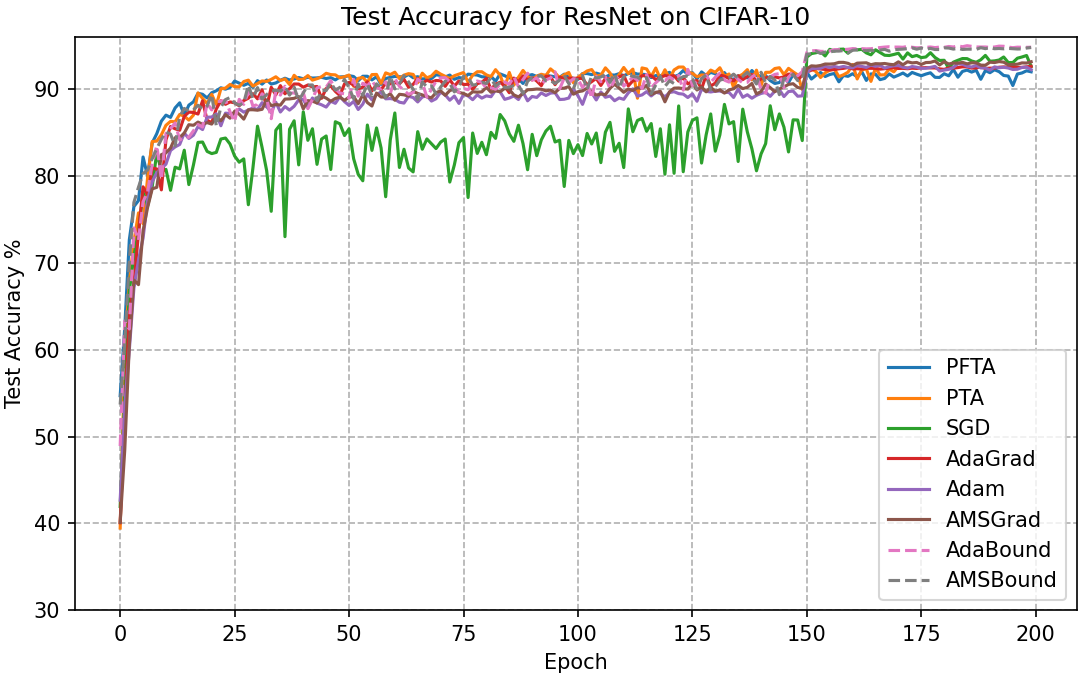}
        \caption{Testing Accuracy for ResNet on CIFAR-10 dataset}
        \label{fig:Exp-2-Resnet-testing}
    \end{figure}
\section{Conclusion}
The dynamics of the different stages of the gradient flow of the GD are carefully analyzed in this paper. Four adaptive learning rate schemes are discussed in terms of convergence speed, analyzing the gradient flow of GD in function of terminal sliding mode and terminal attractor theory. Their performance is analyzed theoretically. The total running times are investigated. The effectiveness of the methods is evaluated through various simulation results for a function approximation problem and an image classification problem. Further work will be carried out in the direction of designing and analyzing other learning rates based on TA and FTA models to further improve the performance of the GD and SGD algorithms. In the course of our research, we find that the direction of the jump of the gradient flow near the local minimum is still random. How to control the direction of the jump and how to control the weights to enter the neighborhood of the global minimum is important. In future work, we will also focus on the research of optimization algorithms based on the terminal attractor mechanism on large-scale machine learning problems.

\section{Acknowledgement}
This work is supported by the National Natural Science Foundation of China (Program Numbers 62176210, U20B2050); the Key Laboratory Program of the Education Department of Shaanxi Province: (Program Numbers 18JS076).

\section{References}
\bibliographystyle{IEEEtran}
\bibliography{PFTA-ARXIV}

\end{document}